\documentclass[journal]{IEEEtran}

\usepackage{cite}

\ifCLASSINFOpdf
 \usepackage[pdftex]{graphicx}
\else
\fi

\usepackage[cmex10]{amsmath}
\usepackage{amsthm,amssymb}
\usepackage{mathtools}
\usepackage{dsfont}

\usepackage{enumitem}

\usepackage{algorithmic}
\usepackage{algorithm}

\ifCLASSOPTIONcompsoc
  \usepackage[caption=false,font=normalsize,labelfont=sf,textfont=sf]{subfig}
\else
  \usepackage[caption=false,font=footnotesize]{subfig}
\fi

\usepackage{url}

\usepackage{xcolor}

\usepackage{hyperref}

\hypersetup{
    colorlinks,
    linkcolor={red!50!black},   
    citecolor={blue!50!black},  
    urlcolor={blue!80!black}    
}

\newcommand{\nsamp}{N} 
\newcommand{\tensordim}{T} 
\newcommand{\sigdim}{p}

\newcommand{\coeff}{\boldsymbol{\alpha}}
\newcommand{\sig}{\mathbf{s}}
\newcommand{\OP}{\mathbf{\Omega}}
\newcommand{\Db}{\mathbf{D}}
\newcommand{\omegab}{\boldsymbol{\omega}}

\newcommand{\Sig}{\mathbf{S}}
\newcommand{\Coeff}{\mathbf{A}}

\newcommand{\Tsig}{\mathcal{S}}
\newcommand{\Tcoeff}{\mathcal{A}}

\newcommand{\Xib}{\boldsymbol{\Xi}}
\newcommand{\xib}{\boldsymbol{\xi}}

\newcommand{\Fcal}{\mathcal{F}}

\newcommand{\Xcal}{\mathcal{X}}

\newcommand{\Cfrak}{\mathfrak{C}}
\newcommand{\Ffrak}{\mathfrak{F}}
\newcommand{\Xfrak}{\mathfrak{X}}

\newcommand{\EE}{\mathbb{E}}
\newcommand{\PP}{\mathbb{P}}
\newcommand{\RR}{\mathbb{R}}

\newcommand{\empRad}[1]{\hat{R}_{\Sig}(#1)}
\newcommand{\Rad}[1]{R_{\nsamp}(#1)}
\newcommand{\empGauss}[1]{\hat{G}_{\Sig}(#1)}

\newcommand{\OB}{\operatorname{Ob}}

\newcommand{\tr}{\operatorname{tr}}
\newcommand{\vect}{\operatorname{vec}}

\newcommand{\mybr}[1]{\{#1\}}

\newtheorem{thm}{Theorem}
\newtheorem{lem}[thm]{Lemma}
\newtheorem{cor}[thm]{Corollary}

\theoremstyle{definition}
\newtheorem{defi}[thm]{Definition}

\theoremstyle{remark}
\newtheorem{nte}{Remark}

\newcommand{\coeffdim}{m}

\begin{document}
%
\title{Learning Co-Sparse Analysis Operators with Separable Structures}
%
%
%

\author{Matthias~Seibert,~\IEEEmembership{Student Member,~IEEE,}
        Julian~W\"ormann,~\IEEEmembership{Student Member,~IEEE,}
        R\'emi~Gribonval,~\IEEEmembership{Fellow,~IEEE,}
        and Martin~Kleinsteuber,~\IEEEmembership{Member,~IEEE}
\thanks{Copyright (c) 2015 IEEE. Personal use of this material is permitted. However, permission to use this material for any other purposes must be obtained from the IEEE by sending a request to pubs-permissions@ieee.org.}
\thanks{The first two authors contributed equally to this work. This paper was presented in part at SPARS 2015, Cambridge, UK.}
\thanks{M. Seibert, J. W{\"o}rmann, and M. Kleinsteuber are with the Department of Electrical and Computer Engineering, TU M{\"u}nchen, Munich, Germany.\newline
E-mail: \{m.seibert,julian.woermann,kleinsteuber\}@tum.de \newline
Web: http://www.gol.ei.tum.de/}
\thanks{R. Gribonval heads the PANAMA project-team (Inria \& CNRS), Rennes, France.
E-mail: remi.gribonval@inria.fr}
}

%
%

\markboth{IEEE Transactions on Signal Processing}%
{Seibert \MakeLowercase{\textit{et al.}}: Learning Co-Sparse Analysis Operators}
%

\maketitle

\begin{abstract}
In the co-sparse analysis model a set of filters is applied to a signal out of the signal class of interest yielding sparse filter responses. As such, it may serve as a prior in inverse problems, or for structural analysis of signals that are known to belong to the signal class. The more the model is adapted to the class, the more reliable it is for these purposes. The task of learning such operators for a given class is therefore a crucial problem. In many applications, it is also required that the filter responses are obtained in a timely manner, which can be achieved by filters with a separable structure.

Not only can operators of this sort be efficiently used for computing the filter responses, but they also have the advantage that less training samples are required to obtain a reliable estimate of the operator.

The first contribution of this work is to give theoretical evidence for this claim by providing an upper bound for the sample complexity of the learning process. The second is a stochastic gradient descent (SGD) method designed to learn an analysis operator with separable structures, which includes a novel and efficient step size selection rule. Numerical experiments are provided that link the sample complexity to the convergence speed of the SGD algorithm.
\end{abstract}


%
\IEEEpeerreviewmaketitle

\section{Introduction}
 \IEEEPARstart{T}{he} ability to sparsely represent signals has become standard practice in signal processing over the last decade. The commonly used synthesis approach has been extensively investigated and has proven its validity in many applications. Its closely related counterpart, the co-sparse analysis approach, was at first not treated with as much interest. In recent years this has changed and more and more work regarding the application and the theoretical validity of the co-sparse analysis model has been published. Both models assume that the signals $\sig$ of a certain class are (approximately) contained in a union of subspaces. In the synthesis model, this reads as
 \begin{equation}
 \sig \approx \Db\mathbf{x}, \quad \mathbf{x} \text{ is sparse.}
 \end{equation}
 In other words, the signal is a linear combination of a few columns of the synthesis dictionary $\Db$. The subspace is determined by the indices of the non-zero coefficients of $\mathbf{x}$.

Opposed to that is the co-sparse analysis model
 \begin{equation}
 \OP \sig \approx \coeff, \quad \coeff \text{ is sparse.}
 \end{equation}
$\OP$ is called the \emph{analysis operator} and its rows represent filters that provide sparse responses. Here, the indices of the filters with zero response determine the subspace to which the signal belongs. This subspace is in fact the intersection of all hyperplanes to which these filters are normal vectors. Therefore, the information of a signal is encoded in its zero responses. In the following, $\coeff$ is referred to as the \emph{analyzed signal}.

While analytic analysis operators like the fused Lasso \cite{tibshirani2005fused} and the finite differences operator, a close relative to the total variation \cite{rudin1992TV}, are frequently used, it is a well known fact that a concatenation of filters which is \emph{adapted to a specific class of signals} produces sparser signal responses. Learning algorithms aim at finding such an optimal analysis operator by minimizing the average sparsity over a representative set of training samples. An overview of recently developed analysis operator learning schemes is provided in Section~\ref{sec:related}.
 
Once an appropriate operator has been chosen there is a plethora of applications that it can be used for. Among these applications are regularizing inverse problems in imaging, cf.~\cite{hawe:tip13, yaghoobi2013constrained, chen2013, chen2014}, where the co-sparsity is used to perform standard task such as image denoising or inpainting,  bimodal super-resolution and image registration as presented in \cite{kiechle2014}, where the joint sparsity of analyzed signals from different modalities is minimized, image segmentation as investigated in \cite{nieuwenhuis2014co}, where structural similarity is measured via the co-sparsity of the analyzed signals, classification as proposed in \cite{shekhar2014analysis}, where an SVM is trained on the co-sparse coefficient vectors of a training set, blind compressive sensing, cf.~\cite{woermann:spl13}, where a co-sparse analysis operator is learned adaptively during the reconstruction of a compressively sensed signal,
and finally applications in medical imaging, e.g., for structured representation of EEG signals~\cite{albera2014brain} and tomographic reconstruction~\cite{pfister2014tomographic}. All these applications rely on the sparsity of the analyzed signal, and thus their success depends on how well the learned operator is adapted to the signal class.
 
An issue commonly faced by learning algorithms is that their performance rapidly decreases as the signal dimension increases.
To overcome this issue some of the authors proposed separable approaches for both dictionary learning, cf.~\cite{hawe2013separable}, and co-sparse analysis operator learning, cf.~\cite{seiwoe14separable}. These separable approaches offer the advantage of a noticeably reduced numerical complexity. 
For example, for a separable operator for image patches of size $p \! \times \! p$ the computational burden for both learning and applying the filters is reduced from $\mathcal{O}(p^2)$ to $\mathcal{O}(p)$. 
We refer the reader to our previous work in \cite{seiwoe14separable} for a detailed introduction of separable co-sparse analysis operator learning.

 In the paper at hand we show that separable analysis operators provide the additional benefit of requiring less samples during the training phase in order to learn a reliable operator. 
This is expressed via the sample complexity for which we provide a result for analysis operator learning, i.e., an upper bound $\eta$ on the deviation of the expected co-sparsity w.r.t.\ the sample distribution and the average co-sparsity of a training set.
Our main result presented in Theorem~\ref{thm:main_result_SC} in Section~\ref{sec:sc} states that $\eta \propto C/\sqrt{\nsamp}$, where  ${\nsamp}$ is the number of training samples. The constant $C$ depends on the constraints imposed on $\OP$ and we show that it is considerably smaller in the separable case.
As a consequence, we are able to provide a generalization bound of an empirically learned analysis operator.

This generalization bound plays a crucial role in the investigation of stochastic gradient descent methods. In Section~\ref{sec:stochastic} we introduce a geometric Stochastic Gradient Descent learning scheme for separable co-sparse analysis operators with a new variable step size selection that is based on the Armijo condition.
The novel learning scheme is evaluated in  Section~\ref{sec:exp}. 
 Our experiments confirm the theoretical results  on sample complexity in the sense that separable analysis operator learning shows an improved convergence rate in the test scenarios.

\section{Notation}
 Scalars are denoted by lower-case and upper-case letters $\alpha,n,N$, column vectors are written as small bold-face letters $\coeff, \sig$, matrices correspond to bold-face capital letters $\Coeff, \Sig$, and tensors are written as calligraphic letters $\Tcoeff, \Tsig$. This notation is consistently used for the lower parts of the structures. For example, the $i^\mathrm{th}$ column of the matrix $\mathbf{X}$ is denoted by $\mathbf{x}_i$, the entry in the $i^\mathrm{th}$ row and the $j^\mathrm{th}$ column of  $\mathbf{X}$ is symbolized by $x_{ij}$, and for tensors $x_{i_1 i_2 \ldots i_{\tensordim}}$ denotes the entry in $\Xcal$ with the indices $i_j$ indicating the position in the respective mode.
 Sets are denoted by blackletter script $\mathfrak{F}, \mathfrak{S}, \mathfrak{X}$.
 
 For the discussion of multidimensional signals, we make use of the operations introduced in \cite{tensor:lathauwer:2000}. In particular, to define the way in which we apply the separable analysis operator to a signal in tensor form we require the $k$-mode product.
 
 \begin{defi}
    Given the $\tensordim$-tensor $\Tsig \in \RR^{I_1 \times I_2 \times \ldots \times I_{\tensordim}}$ and the matrix $\OP \in \RR^{J_k \times I_k}$, their $k$-mode product is denoted by 
    \begin{equation*}
    \Tsig \times_k \OP.
    \end{equation*}
    The resulting tensor is of the size $I_1 \times I_2 \times \ldots \times I_{k-1} \times J_k \times I_{k+1} \times \ldots \times I_{\tensordim}$ and its entries are defined as
    \begin{equation*}
        (\Tsig \times_k \OP)_{i_1 i_2 \ldots i_{k-1} j_k i_{k+1} \ldots i_{\tensordim}}
        = \sum_{i_k = 1}^{I_k} s_{i_1 i_2 \ldots i_{\tensordim}} \cdot \omega_{j_k i_{k}}
    \end{equation*}
    for $j_k = 1,\ldots,J_k$.
 \end{defi}
The $k$-mode product can be rewritten as a matrix-vector product using the Kronecker product $\otimes$ and the $\vect$-operator that rearranges a tensor into a column vector such that
\begin{equation}
\begin{split}
\label{eq:vectorized_nmode}
    \Tcoeff &= \Tsig \times_1 \OP^{(1)} \times_2 \OP^{(2)} \ldots \times_{\tensordim} \OP^{(\tensordim)}\\
    \Leftrightarrow \vect(\Tcoeff) &= (\OP^{(1)} \otimes \OP^{(2)} \otimes \ldots \otimes \OP^{(\tensordim)}) \cdot \vect(\Tsig).
\end{split}
\end{equation}
We also make use of the mapping 
\begin{equation}\begin{split}
\label{eq:sep_mapping}
    \iota\colon &\RR^{J_1 \times I_1} \times \ldots \times \RR^{J_\tensordim \times I_ \tensordim} \to \RR^{\prod_k J_k \times \prod_k I_k}\\
    {}& (\OP^{(1)},\ldots ,\OP^{(\tensordim)}) \mapsto \OP^{(1)} \otimes \ldots \otimes \OP^{(\tensordim)}.
\end{split}\end{equation}

The remainder of notational comments, in particular those required for the discussion of the sample complexity, will be provided in the corresponding sections.

\section{Related Work}
\label{sec:related}

As we have pointed out, learning an operator adapted to a class of signals yields a sparser representation than those provided by analytic filter banks. It thus comes as no surprise that there exists a variety of analysis operator learning algorithms, which we shortly review in the following.

In \cite{rubinstein2013analysis} the authors present an adaptation of the well known K-SVD dictionary learning algorithm to the co-sparse analysis operator setting. The training phase consists of two stages. In the first stage the rows of the operator that determine the subspace that each signal resides in are determined. In the subsequent stage each row of the operator is updated to be the vector that is ``most orthogonal'' to the signals associated with it. These two stages are repeated until a convergence criterion is met.

In \cite{yaghoobi2013constrained} it is postulated that the analysis operator is a uniformly normalized tight frame, i.e., the columns of the operator are orthogonal to each other while all rows have the same $\ell_2$-norm. Given noise contaminated training samples, an algorithm is proposed that outputs an analysis operator as well as noise free approximations of the training data. This is achieved by an alternating two stage optimization algorithm. In the first stage the operator is updated using a projected subgradient algorithm, while in the second stage the signal estimation is updated using alternating direction method of multipliers (ADMM).

A concept very similar to that of analysis operator learning is called sparsifying transform learning. In \cite{ravishankar2013} a framework for learning overcomplete sparsifying transforms is presented. This algorithm consists of two steps, a sparse coding step where the sparse coefficient is updated by only retaining the largest coefficients, and a transform update step where a standard conjugate gradient method is used and the resulting operator is obtained by normalizing the rows.

The authors of \cite{chen2014} propose a method specialized on image processing. Instead of a patch-based approach, an image-based model is proposed with the goal of enforcing coherence across overlapping patches. In this framework, which is based on higher-order filter-based Markov Random Field models, all possible patches in the entire image are considered at once during the learning phase.
A bi-level optimization scheme is proposed that has at its heart an unconstrained optimization problem w.r.t.\ the operator, which is solved using a quasi-Newton method.

Dong \emph{et al.}~\cite{dong2014} propose a method that alternates between a hard thresholding operation of the co-sparse representation and an operator update stage where all rows of the operator are simultaneously updated using a gradient method on the sphere. Their target function has the form $\|\Coeff - \OP \Sig\|_F^2$, where $\Coeff$ is the sparse representation of the signal $\Sig$.

Finally, Hawe \emph{et al.}~\cite{hawe:tip13} propose a geometric conjugate gradient algorithm on the product of spheres, where analysis operator properties like low coherence and full rank are incorporated as penalty functions in the learning process.

Except for our previous work~\cite{seiwoe14separable}, to our knowledge the only other analysis operator learning approach that offers a separable structure is proposed in \cite{qi:icip13separable} for the two-dimensional setting. Therein, an algorithm is developed that takes as an input noisy 2D images $\hat{\Sig}_i,\, i=1,\ldots,\nsamp$ and then attempts to find $\Sig_i,\OP_1,\OP_2$ that minimize $\sum_{i=1}^{\nsamp} \|\Sig_i - \hat{\Sig}_i\|_F^2$ such that $\|\OP_1 \Sig_i \OP_2^\top\|_0 \leq l$, where $l$ is a positive integer that serves as an upper bound on the number of non-zero entries, and the rows of $\OP_1$ and $\OP_2$ have unit norm. This problem is solved by alternating between a sparse coding stage and an operator update stage that is inspired by the work in \cite{rubinstein2013analysis} and relies on singular value decompositions.

While, to our knowledge, there are no sample complexity results for separable co-sparse analysis operator learning, results for many other matrix factorization schemes exist. Examples for this can be found in \cite{maurer2010,vainsencher2011,gribonval2013sample}, among others. 
Specifically, \cite{gribonval2013sample} provides a broad overview of sample complexity results for various matrix factorizations. 
It is also of particular interest to our work since the sample complexity of separable dictionary learning is discussed. The bound derived therein has the form $c \sqrt{\beta \log(N) / N}$ where the driving constant $\beta$ behaves proportional to $\sum_i \sigdim_i d_i$ for multidimensional data in $\RR^{\sigdim_1 \times \ldots \times \sigdim_{\tensordim}}$ and dictionaries $\Db^{(i)} \in \RR^{p_i \times d_i}$. This is an improvement over the non-separable result where the driving constant is proportional to the product over all $i$, i.e., $\beta \propto \prod_i \sigdim_i d_i$. The argumentation used to derive the results in \cite{gribonval2013sample} is different from the one we employ throughout this paper. While the results in \cite{gribonval2013sample} are derived by determining a Lipschitz constant and then using an argument based on covering numbers and concentration of measure, we follow a different approach. Following the work in \cite{maurer2010}, we employ McDiarmid's inequality in combination with a symmetrization argument and Rademacher averages.
This approach offers better results when discussing tall matrices as in the case of co-sparse analysis operator learning.

\section{Sample Complexity}
\label{sec:sc}
Co-sparse analysis operator learning aims at finding a set of filters concatenated in an operator $\OP$ which generates an optimal sparse representation $\Coeff = \OP \Sig$ of a set of training samples $\Sig = [\sig_1,\ldots,\sig_\nsamp]$. This is achieved by solving the optimization problem
\begin{equation}
 \label{eq:def_sparsepen}
 \arg\min_{\OP \in \Cfrak} \tfrac{1}{\nsamp} \sum_{j=1}^{\nsamp} f(\OP,\sig_j)
\end{equation}
where $f(\OP,\sig) = g(\OP\sig) + p(\OP)$ with the sparsity promoting function $g$ and the penalty function $p$. By restricting $\OP$ to the constraint set $\Cfrak$ it is ensured that certain trivial solutions are avoided, see e.g. \cite{yaghoobi2013constrained}. The additional penalty function is used to enforce more specific constraints. We will discuss appropriate choices of constraint sets at a later point in this section while the penalty function will be concretized in Section~\ref{sec:stochastic}. 
%

Before we can provide our main theoretical result, we first introduce several concepts from the field of statistics in order to make this work self-contained.

\subsection{Rademacher \& Gaussian Complexity}
\label{sec:RuGC}
In the following, we consider the set of samples $\Sig = [\sig_1, \ldots, \sig_\nsamp]$, $\sig_i \in \Xfrak$, where each sample is drawn according to an underlying distribution $\PP$ over $\Xfrak$. Furthermore, given the above defined function $f\colon \Cfrak \times \Xfrak \to \RR$, we consider the class $\Ffrak = \{f(\OP,\cdot)\,:\, \OP \in \Cfrak\}$ of functions that map the sample space $\Xfrak$ to $\RR$.
We are interested in finding the function $f \in \Ffrak$ for which the expected value 
\begin{align*}
    \EE[f] &\coloneqq \EE_{\sig \sim \PP} [f(\OP,\sig)]
\end{align*}
is minimal. However, due to the fact that the distribution $\PP$ of the data samples is not known, in general, it is not possible to determine the optimal solution to this problem and we are limited to finding a minimizer of the empirical mean for a given set of $\nsamp$ samples $\Sig$ drawn according to the underlying distribution. The empirical mean is defined as
\begin{align*}
    \hat{\EE}_{\Sig}[f] &\coloneqq \tfrac{1}{\nsamp}\sum_{i=1}^{\nsamp} f(\OP,\sig_i).
\end{align*}

In order to evaluate how well the empirical problem approximates the expectation we pursue an approach that relies on the Rademacher complexity. We use the definition introduced in \cite{meir2003generalization}.

\begin{defi}
\label{def:empRad}
Let $\Ffrak \subset \{ f(\OP,\cdot) \,:\, \OP \in \Cfrak \}$ be a family of real valued functions defined on the set $\Xfrak$.
Furthermore, let $\Sig = [\sig_1,\ldots,\sig_\nsamp]$ be a set of samples with $\sig_i \in \Xfrak$.
The \emph{empirical Rademacher complexity} of $\Ffrak$ with respect to the set of samples $\Sig$ is defined as
\begin{equation*}
    \empRad{\Ffrak} \coloneqq \EE_\sigma \left[\sup_{f \in \Ffrak} \tfrac{1}{\nsamp}\sum_{i=1}^\nsamp \sigma_i f(\OP,\sig_i)\right]
\end{equation*}
where $\sigma_1,\ldots,\sigma_{\nsamp}$ are independent Rademacher variables, i.e., random variables with $Pr(\sigma_i = +1) = Pr(\sigma_i = -1) = 1/2$ for $i=1,\ldots,\nsamp$. 
\end{defi}

This definition differs slightly from the standard one, 
where the absolute value of the argument within the supremum is taken, cf.\ \cite{bartlett2003rademacher}. 
Both definitions coincide when $\Ffrak$ is closed under negation, i.e., when  $f \in \Ffrak$ implies $-f \in \Ffrak$. As proposed in~\cite{meir2003generalization}, the definition of the empirical Rademacher complexity as above has the property that it is dominated by the standard empirical Rademacher complexity. Furthermore, $\empRad{\Ffrak}$ vanishes when the function class $\Ffrak$ consists of a single constant function.

Definition~\ref{def:empRad} is based on a fixed set of training samples $\Sig$. However, we are generally interested in the correlation of $\Ffrak$ with respect to a distribution $\PP$ over $\Xfrak$. This encourages the following definition.

\begin{defi}
Let $\Ffrak$ be as before and $\Sig = [\sig_1,\ldots,\sig_\nsamp]$ be a set of samples
$\sig_i,\, i=1,\ldots,\nsamp$ drawn i.i.d. according to a predefined probability distribution $\PP$. Then the \emph{Rademacher complexity} of $\Ffrak$ is defined as
\begin{equation*}
    \Rad{\Ffrak} \coloneqq \EE_{\Sig}[\empRad{\Ffrak}].
\end{equation*}
\end{defi}

With these definitions it is possible to provide generalization bounds for general function classes. Examples for this can be found in \cite{bartlett2003rademacher}. 
In addition to the Rademacher complexity, another measure of complexity is required to obtain bounds for our concrete case at hand. As before, the definition used here slightly differs from the standard definition, which can be found in \cite{bartlett2003rademacher}.

\begin{defi}
Let $\Ffrak \subset \{ f(\OP,\cdot) \,:\, \OP \in \Cfrak \}$ be a family of real valued functions defined on the set $\Xfrak$.
Furthermore, let $\Sig = [\sig_1,\ldots,\sig_\nsamp]$ be a set of samples with $\sig_i \in \Xfrak$.
Then the \emph{empirical Gaussian complexity} of the function class $\Ffrak$ is defined as
\begin{equation*}
    \hat{G}_{\Sig}(\Ffrak) = \EE_\gamma \left[ \sup_{f \in \Ffrak} \tfrac{1}{\nsamp} \sum_{i=1}^\nsamp \gamma_i f(\OP,\sig_i) \right]
\end{equation*}
where $\gamma_1,\ldots,\gamma_{\nsamp}$ are independent Gaussian $\mathcal{N}(0,1)$ random variables. 
The \emph{Gaussian complexity} of $\Ffrak$ is defined as
\begin{equation*}
    G_\nsamp(\Ffrak) = \EE_{\Sig}\left[ \hat{G}_{\Sig}(\Ffrak) \right].
\end{equation*}
\end{defi}

Based on the similar construction of Rademacher and Gaussian complexity it is not surprising that it is possible to prove that they fulfill a similarity condition. For us, it is only of interest to upper bound the Rademacher complexity with the Gaussian complexity.
\begin{lem}
\label{lem:rad_gauss}
Let $\Ffrak$ be a class of functions mapping from $\Xfrak$ to $\RR$. 
For any set of samples $\Sig$, the empirical Rademacher complexity can be upper bounded with the empirical Gaussian complexity via
\begin{equation*}
    \empRad{\Ffrak} \leq \sqrt{\pi/2} \cdot \empGauss{\Ffrak}.
\end{equation*}
\end{lem}

This can be seen by noting that $\EE[|\gamma_i|] = \sqrt{2/\pi}$ and by using Jensen's inequality, cf.~\cite{gribonval2014}.

\subsection{Generalization Bound for Co-Sparse Analysis Operator Learning}

In this section we provide the concrete bounds for sample complexity of co-sparse analysis operator learning.
At the beginning of Section~\ref{sec:sc} we briefly mentioned that the key role of constraint sets is to avoid trivial solutions \cite{yaghoobi2013constrained}. A very simple constraint, that achieves this goal is to require that each row of the learned operator has unit $\ell_2$-norm, cf.~\cite{hawe:tip13,yaghoobi2013constrained}. The set of all matrices that fulfills this property has a manifold structure and is often referred to as the \emph{oblique manifold}
\begin{equation}
\label{eq:def_obl}
    \OB(\coeffdim,\sigdim) = \{\OP \in \mathbb{R}^{\coeffdim \times \sigdim} \,:\, (\OP\OP^\top)_{ii} = 1,\, i=1,\ldots,\coeffdim\}.
\end{equation}
This is the constraint set we employ for \emph{non-separable} operator learning, which we want to distinguish from learning operators with separable structure.

A separable structure on the operator is enforced by further restricting the constraint set to the subset $\{ \OP \in \OB(\coeffdim, \sigdim) \,:\, \OP = \iota(\OP^{(1)},\ldots,\OP^{(\tensordim)}),\, \OP^{(i)} \in \OB(\coeffdim_i, \sigdim_i) \}$ with the appropriate dimensions $\coeffdim = \prod_i \coeffdim_i$ and $\sigdim = \prod_i \sigdim_i$. The mapping $\iota$ is defined in Equation~\eqref{eq:sep_mapping}. The fact that $\iota(\OP^{(1)},\ldots,\OP^{(\tensordim)})$ is an element of $\OB(\coeffdim, \sigdim)$ is readily checked. 
While $\iota$ is not bijective onto $\OB(\coeffdim, \sigdim)$, this does not pose a problem for our scenario. 
This way of expressing separable operators is related to signals $\Tsig$ in tensor form via 
\begin{equation*}
\iota(\OP^{(1)}, \ldots, \OP^{(\tensordim)}) \sig = \vect^{-1}(\sig) \times_1 \OP^{(1)} \ldots \times_{\tensordim} \OP^{(\tensordim)}
\end{equation*}
with $\vect^{-1}(\sig) = \Tsig$.

To provide concrete results we will require the ability to bound the absolute value of the realization of a function to its expectation. We use McDiarmid's inequality, cf.~\cite{mcdiarmid89}, to tackle this task. 
\begin{thm}[McDiarmid's Inequality]
\label{thm:mcdiarmid}
    Suppose $X_1,\ldots,X_\nsamp$ are independent random variables taking values in a set $\Xfrak$ and assume that $f:\Xfrak^\nsamp \to \RR$ satisfies
    \begin{equation*}\begin{split}
        \sup_{x_1,\ldots,x_{\nsamp}, \hat{x}_i}\!\!\!\! &|f(x_1,\ldots,x_\nsamp)
        - f(x_1,\ldots,x_{i-1},\hat{x}_i,x_{i+1},\ldots, x_\nsamp)|\\ 
        {}&\leq c_i \qquad\qquad\text{for } 1\leq i \leq \nsamp.
    \end{split}\end{equation*}
    It follows that for any $\varepsilon > 0$
    \begin{equation*}\begin{split}
        Pr&(\EE[f(X_1, \ldots, X_{\nsamp})] - f(X_1, \ldots, X_{\nsamp}) \geq \varepsilon)\\
        {}&\leq \exp\left(-\frac{2 \varepsilon^2}{\sum_{i=1}^{\nsamp} c_i^2}\right).
    \end{split}\end{equation*}
\end{thm}

We are now ready to state a preliminary result. 

\begin{lem}
\label{lem:dist_with_Gauss}
Let $\Sig = [ \sig_1, \ldots, \sig_\nsamp ]$ be a set of samples independently drawn according to a distribution within the unit $\ell_2$-ball in $\RR^{\sigdim}$.
Let $f(\OP,\sig) = g(\OP\sig) + p(\OP)$ as previously defined where the sparsity promoting function $g$ is $\lambda$-Lipschitz. Finally, let the function class $\Ffrak$ be defined as $\Ffrak = \{ f(\OP,\cdot)\,:\, \OP \in \OB(\coeffdim, \sigdim)\}$. Then the inequality

\begin{equation}
\label{eq:gbound_gauss}
    \EE[f] - \hat{\EE}_{\Sig}[f] \leq \sqrt{2\pi}\ \empGauss{\Ffrak} + 3\sqrt{\frac{2 \lambda^2 \coeffdim \ln(2/\delta)}{\nsamp}},
\end{equation}
holds with probability greater than $1-\delta$.
\end{lem}
\begin{proof}
When considering the difference $\EE[f] - \hat{\EE}_{\Sig}[f]$ the penalty function $p$ can be omitted since it is independent of the samples and therefore cancels out.
Now, in order to bound the difference $\EE[f] - \hat{\EE}_{\Sig}[f]$ for all $f \in \Ffrak$, we consider the equivalent problem of bounding $\sup_{f \in \Ffrak} (\EE[f] - \hat{\EE}_{\Sig}[f])$. To do this we introduce the random variable
\begin{equation*}
    \Phi(\Sig) = \sup_{f \in \Ffrak} (\EE[f] - \hat{\EE}_{\Sig}[f]).
\end{equation*}


The next step is to use McDiarmid's inequality to bound $\Phi(\Sig)$. Since $\OP$ is an element of the constraint set $\OB(\coeffdim,\sigdim)$, its largest singular value is bounded by $\sqrt{\coeffdim}$. Furthermore, due to the assumptions  that $g$ is $\lambda$-Lipschitz and $\|\sig_i\|_2 \leq 1$, the function value of $f(\OP,\sig)$ changes by at most $2\lambda\sqrt{\coeffdim}$ when varying $\sig$. Since $f$ appears within the empirical average in $\Phi$, we get the result that the function value of $\Phi$ varies by at most $2\lambda\sqrt{\coeffdim} / \nsamp$ when changing a single sample in the set $\Sig$.
Thus, McDiarmid's inequality stated in Theorem~\ref{thm:mcdiarmid} with a target probability of $\delta$ yields the bound
\begin{equation}
\label{eq:lem_7_proof_1}
    \Phi(\Sig) \leq \EE_{\Sig}[\Phi(\Sig)] + \sqrt{\frac{2 \lambda^2 \coeffdim \ln(1/\delta)}{\nsamp}}
\end{equation}
with probability greater than $1-\delta$.
By using a standard symmetrization argument, cf.~\cite{mendelson2003few}, and another instance of McDiarmid's inequality we can then first upper bound $\EE_{\Sig}[\Phi(\Sig)]$ by $2\Rad{\Ffrak}$ and then by $2\empRad{\Ffrak}$, yielding
\begin{equation*}
    \sup_{f \in \Ffrak} (\EE[f] - \hat{\EE}_{\Sig}[f]) \leq 2\empRad{\Ffrak} + 3\sqrt{\frac{2 \lambda^2 \coeffdim \ln(2/\delta)}{\nsamp}}
\end{equation*}
with probability greater than $1-\delta$. A more detailed derivation of these bounds can be found in the appendix. Lemma~\ref{lem:rad_gauss} then provides the proposed bound.
\end{proof}

The last ingredient for the proof of our main theorem is Slepian's Lemma, cf.~\cite{ledoux2013}, which is used to provide an estimate for the expectation of the supremum of a Gaussian process. 
\begin{lem}[Slepian's Lemma]
\label{thm:slepian}
    Let $X$ and $Y$ be two centered Gaussian random vectors in $\RR^{N}$ such that
    \begin{equation*}
        \EE[|Y_i - Y_j|^2] \leq \EE[|X_i - X_j|^2]\quad \text{ for } i \neq j.
    \end{equation*}
    Then
    \begin{equation*}
        \EE\left[ \sup_{1 \leq 1 \leq N} Y_i \right] \leq \EE\left[ \sup_{1 \leq i \leq N} X_i \right].
    \end{equation*}
\end{lem}

With all the preliminary work taken care of we are now able to state and prove our main results.

\begin{thm}
\label{thm:main_result_SC}
Let $\Sig = [ \sig_1, \ldots, \sig_\nsamp ]$ be a set of samples independently drawn according to a distribution within the unit $\ell_2$-ball in $\RR^{\sigdim}$.
Let $f(\OP,\sig) = g(\OP\sig) + p(\OP)$ as previously defined where the sparsity promoting function $g$ is $\lambda$-Lipschitz. Finally, let the function class $\Ffrak$ be defined as $\Ffrak = \{ f(\OP,\cdot)\,:\, \OP \in \Cfrak\}$, where $\Cfrak$ is either $\OB(\coeffdim,\sigdim)$ for the non-separable case or the subset $\{ \OP \in \OB(\coeffdim, \sigdim) \,:\, \OP = \iota(\OP^{(1)},\ldots,\OP^{(\tensordim)}),\, \OP^{(i)} \in \OB(\coeffdim_i, \sigdim_i) \}$ for the separable case. Then we have
\begin{equation}
    \EE[f] - \hat{\EE}_{\Sig}[f] \leq \sqrt{2\pi}\, \frac{\lambda C_\Cfrak}{\sqrt{\nsamp}} + 3\sqrt{\frac{2 \lambda^2 \coeffdim \ln(2/\delta)}{\nsamp}}
\end{equation}
with probability at least $1 -\delta$, where $C_\Cfrak$ is a constant that depends on the constraint set. In the non-separable case the constant is defined as $C_\Cfrak = \coeffdim\sqrt{\sigdim}$, whereas in the separable case it is given as $C_\Cfrak = \sum_i \coeffdim_i \sqrt{\sigdim_i}$.
\end{thm}
\begin{proof}
Given the results of Lemma~\ref{lem:dist_with_Gauss} it remains to provide bounds for the empirical Gaussian complexity. We discuss the two considered constraint sets separately in the following.

\emph{Non-Separable Operator:}
In order to find a bound for $\empGauss{\Ffrak}$ we define the two Gaussian processes $G_\OP = \tfrac{1}{\nsamp}\sum_{i=1}^\nsamp \gamma_i f(\OP,\sig_i)$ and $H_\OP = \tfrac{\lambda}{\sqrt{\nsamp}} \langle \Xib , \OP \rangle_F = \tfrac{\lambda}{\sqrt{\nsamp}} \sum_{i,j} \xi_{ij}\omega_{ij}$ with $\gamma_i$ and $\xi_{ij}$ i.i.d.\ Gaussian random variables. These two processes fulfill the condition
\begin{equation*}
    \EE_\gamma[|G_{\OP} - G_{\OP^\prime}|^2] \leq \tfrac{\lambda^2}{\nsamp}\|\OP - \OP^\prime\|_F^2 = \EE_\xi[|H_{\OP} - H_{\OP^\prime}|^2],
\end{equation*}
where the inequality holds since 
$f(\OP,\sig_i)$ is $\lambda$-Lipschitz w.r.t.\ the Frobenius norm in its first component when omitting the penalty term, i.e.,
\begin{equation*}
    |f(\OP,\sig_i) - f(\OP^\prime,\sig_i)| = |g(\OP\sig_i) - g(\OP^\prime \sig_i)| \leq \lambda \|\OP - \OP^\prime\|_F
\end{equation*}
for all $\OP,\OP^\prime \in \Cfrak$. 
Thus, we can apply Slepian's Lemma, cf.~Lemma~\ref{thm:slepian}, which provides the inequality 
\begin{equation}
    \EE_\gamma[\sup_{\OP \in \Cfrak} G_\OP] \leq \EE_\xi[\sup_{\OP \in \Cfrak} H_\OP].
\end{equation}
Note, that the left-hand side of this inequality is the empirical Gaussian complexity of our learning problem. Considering the constraint set $\Cfrak$ the expression on the right-hand side can be bounded via
\begin{equation*}\begin{split}
    \EE_\xi[\sup H_\OP] =& \tfrac{\lambda}{\sqrt{\nsamp}} \EE_\xi[\sup\nolimits_{\OP \in \Cfrak} \langle \Xib,\OP \rangle_F] \\
    =& \tfrac{\lambda}{\sqrt{\nsamp}} \EE_\xi[\sum_{j=1}^{\coeffdim} \|\xib_j\|_2] 
    \leq \tfrac{\lambda}{\sqrt{\nsamp}} \coeffdim \sqrt{\sigdim}.
\end{split}\end{equation*}
Here, $\xib_j \in \mathbb{R}^{\sigdim},\, j=1,\ldots,\coeffdim$ denotes the transposed of the $j$-th row of $\Xib$. 

\emph{Separable Operator:}
To consider the separable analysis operator in the sense of our previous work \cite{seiwoe14separable}, we define the set of functions
\begin{equation*}\begin{split}
    \hat{f}\colon &\Cfrak \times \RR^{\sigdim} \to \RR,\\
    {}&\OP \mapsto g(\iota(\OP) \sig).
\end{split}\end{equation*}
The function $\hat{f}$ operates on the direct product of manifolds $\Cfrak = \OB_1 \times \OB_2 \times \ldots \times \OB_{\tensordim}$ and utilizes the function $\iota$ as defined in Equation~\eqref{eq:sep_mapping}. The signals $\sig_i \in \RR^{\sigdim}$ can be interpreted as vectorized versions of tensorial signals $\Tsig_i \in \RR^{\sigdim_1 \times \ldots \times \sigdim_\tensordim}$ where $\sigdim = \prod \sigdim_i$.
Above, we showed that $f$ is $\lambda$-Lipschitz w.r.t.\ the Frobenius norm on its first variable $\OP$. As $\Cfrak$ is a subset of a large oblique manifold, the same holds true for $\hat{f}$.

Similar to before, we define two Gaussian processes $G_{\OP} = \tfrac{1}{\nsamp}\sum_{i=1}^\nsamp \gamma_i \hat{f}(\OP,\sig_i)$ with $\OP \in \Cfrak$, and $H_{\OP} = \tfrac{\lambda}{\sqrt{\nsamp}} \sum_{i=1}^{\tensordim} \langle \Xib^{(i)}, \OP^{(i)} \rangle_F$.
The expected value $\EE[|H_{\OP} - H_{\OP^\prime}|^2]$ can be equivalently written as $\tfrac{\lambda^2}{\nsamp} \EE[(\sum_{i=1}^{\tensordim} \tr( (\Xib^{(i)})^\top \OP^{(i)}) )^2]$ and the inequality
\begin{equation*}
    \EE_\gamma[|G_{\OP} - G_{\OP^\prime}|^2] \leq \tfrac{\lambda^2}{\nsamp}\|\OP - \OP^\prime\|_F^2 = \EE_\xi[|H_{\OP} - H_{\OP^\prime}|^2]
\end{equation*}
holds, just as in the non-separable case. Hence, we are able to apply Slepian's lemma which yields the inequality $\EE[\sup_{\OP \in \Cfrak} G_\OP] \leq \EE[\sup_{\OP \in \Cfrak} H_\OP]$. It only remains to provide an upper bound for the right-hand side.

Using the fact that $\Cfrak$ is now the direct product of oblique manifolds we get
\begin{align*}
    \EE_\xi&\left[\sup\nolimits_{\OP \in \Cfrak} H_\OP \right] 
    = \EE_\xi \left[\sup\nolimits_{\OP \in \Cfrak} \tfrac{\lambda}{\sqrt{\nsamp}} \sum_{i=1}^{\tensordim} \tr\left((\Xib^{(i)})^\top \OP^{(i)}\right)\right]\\
    {}&= \tfrac{\lambda}{\sqrt{\nsamp}} \sum_{i=1}^{\tensordim} \EE_\xi \left[\sup\nolimits_{\OP^{(i)} \in \OB_i} \tr\left( (\Xib^{(i)})^\top \OP^{(i)} \right) \right]\\
    {}&= \tfrac{\lambda}{\sqrt{\nsamp}} \sum_{i=1}^{\tensordim} \EE_\xi \left[\sum_{j=1}^{\coeffdim_i} \|\xib^{(i)}_j\|_2\right]
    \leq \tfrac{\lambda}{\sqrt{\nsamp}} \sum_{i=1}^{\tensordim} \coeffdim_i \sqrt{\sigdim_i},
\end{align*}
where $\xib_j^{(i)}$ denotes the transposed of the $j$-th row of $\Xib^{(i)}$. The last inequality holds due to Jensen's inequality and the fact that all $\xi_{ij}$ are $\mathcal{N}(0,1)$ random variables.
\end{proof}

\begin{nte}
\label{nte:sc_absolute value}
Theorem~\ref{thm:main_result_SC} can be extended to the absolute value $|\EE[f] - \hat{\EE}_{\Sig}[f]|$ of the deviation by redefining the function class as $\Fcal \cup (-\Fcal)$.
\end{nte}

\section{Stochastic Gradient Descent for Analysis Operator Learning}
\label{sec:stochastic}
Stochastic Gradient Descent (SGD) is particularly suited for large scale optimization and thus a natural choice for many machine learning problems. 

Before we describe a geometric SGD method that respects the underlying constraints on the analysis operator,
we follow the discussion of SGD methods provided by Bottou in \cite{bottou-2010} in order to establish a connection of the excess error and the sample complexity result derived in the previous section. 
Let $f_{}^\star = \arg \min_{f \in \Ffrak} \EE[f]$ be the best possible prediction function, let $f^\star_{\Sig} = \arg \min_{f \in \Ffrak} \hat{\EE}_{\Sig}[f]$ be the best possible prediction function for a set of training samples $\Sig$, and let $\tilde{f}_{\Sig}^{}$ be the solution found by an optimization method with respect to the provided set of samples $\Sig$.
Bottou proposes that the so-called excess error $\mathcal{E} = \EE[\tilde{f}_{\Sig}^{}] - \EE[f_{}^\star]$ can be decomposed as the sum $\mathcal{E} = \mathcal{E}_\text{est} + \mathcal{E}_\text{opt}$.
Here, the estimation error $\mathcal{E}_\text{est} = \EE[f^\star_{\Sig}] - \EE[f_{}^\star]$ measures the distance between the optimal solution for the expectation 
and the optimal solution for the empirical average while the optimization error $\mathcal{E}_\text{opt} = \EE[\tilde{f}_{\Sig}^{}] - \EE[f^\star_{\Sig}]$ quantifies the distance between the optimal solution for the empirical average and the solution obtained via an optimization algorithm. 

While $\mathcal{E}_\text{opt}$ is dependent on the optimization strategy, the estimation error $\mathcal{E}_\text{est}$ is closely related to the previously discussed sample complexity. Lower bounds on the sample complexity also apply to the estimation error as specified in the following Corollary.

\begin{cor} 
Under the same conditions as in Theorem~\ref{thm:main_result_SC} the estimation error is upper bounded by
\begin{equation}
    \mathcal{E}_\text{est} \leq 2 \sqrt{2 \pi}\, \frac{\lambda C_\Cfrak}{\sqrt{\nsamp}} + 6\sqrt{\frac{2 \lambda^2 \coeffdim \ln(2/\delta)}{\nsamp}}
\end{equation}
with probability at least $1 - \delta$.
\end{cor}
\begin{proof} The estimation error can be bounded via
\begin{equation*}\begin{split}
    \mathcal{E}_\text{est} &=  \EE[f^\star_{\Sig}] - \EE[f_{}^\star]
    \leq \EE[f^\star_{\Sig}] - \EE[f_{}^\star] - \hat{\EE}_{\Sig}[f^\star_{\Sig}] + \hat{\EE}_{\Sig}[f_{}^\star] \\
    {}&\leq |\EE[f^\star_{\Sig}] - \hat{\EE}_{\Sig}[f^\star_{\Sig}]| + |\hat{\EE}_{\Sig}[f_{}^\star] - \EE[f_{}^\star]|,
\end{split}\end{equation*}
where the first inequality holds since $f^\star_{\Sig}$ is the minimizer of $\hat{\EE}_{\Sig}$, and therefore $\hat{\EE}_{\Sig}[f^\star_{\Sig}] \leq \hat{\EE}_{\Sig}[_{}f^\star]$ and the final result follows from Theorem~\ref{thm:main_result_SC} and its subsequent remark.
\end{proof}

\subsection{Geometric Stochastic Gradient Descent}
\label{subsec_gsgd}
Ongoing from the seminal work of \cite{robbins1951}, SGD type optimization methods have attracted attention to solve large-scale machine learning problems \cite{bottou-lecun-2004, Mairal2010}. In contrast to full gradient methods that in each iteration require the computation of the gradient with respect to all the $\nsamp$ training samples $\Sig = [\sig_1,\ldots, \sig_{\nsamp}]$ in SGD the gradient computation only involves a small batch randomly drawn from the training set in order to find the $\OP \in \Cfrak$ which minimizes the expectation $\EE_{\sig \sim \PP}[f(\OP,\sig)]$. Accordingly, the cost of each iteration is independent of $\nsamp$ (assuming the cost of accessing each sample is independent of $\nsamp$). 

In the following we use the notation $\sig_{\mybr{k(i)}}$ to denote a signal batch of cardinality $\vert k(i) \vert$, where $k(i)$ represents an index set randomly drawn from $\{1,2,\ldots,\nsamp\}$ at iteration $i$.

In order to account for the constraint set $\Cfrak$ we follow \cite{bonnabel2013} and propose a geometric SGD optimization scheme. This requires some adaptions to classic SGD. 
The subsequent discussion provides a concise introduction to line search optimization methods on manifolds. For more insights into optimization on manifolds in general we refer the reader to \cite{absil2009} and to \cite{bonnabel2013} for optimization on manifolds using SGD in particular.

In Euclidean space the direction of steepest descent at a point $\OP$ is given by the negative (Euclidean) gradient. 
For optimization on an embedded manifold $\Cfrak$ this role is taken over by the negative Riemannian gradient, which is a projection of the gradient onto the respective tangent space $T_{\OP}$. To keep notation simple, we denote the Riemannian gradient w.r.t.\ $\OP$ at a point $(\OP,\sig)$ by $\mathbf{G}(\OP, \sig) = \Pi_{T_{\OP} \Cfrak} (\nabla_{\OP} f(\OP,\sig))$.
Optimization methods on manifolds find a new iteration point by searching along geodesics instead of following a straight path. We denote a geodesic emanating from point $\OP$ in direction $\mathbf{H}$ by $\Gamma(\OP,\mathbf{H},\cdot)$. Following this geodesic for distance $t$ then results in the new point $\Gamma(\OP,\mathbf{H},t) \in \Cfrak$.
Finally, an appropriate step size $t$ has to be computed. A detailed discussion on this topic is provided in Section \ref{sec:step_size}.

Using these definitions, an update step of geometric SGD reads as
\begin{align}
	\OP_{i+1} = \Gamma(\OP_{i},-{\mathbf{G}}(\OP_{i}, \sig_{\mybr{k(i)}}),t_{i}).
	\label{eq:sgd_update}
\end{align}
Since the SGD framework only provides a noisy estimate of the objective function in each iteration, a stopping criterion based on the average over previous iterations is chosen to terminate the optimization scheme. 
First, let $f(\OP_{i},\sig_{\mybr{k(i)}})$ denote the mean over all signals in the batch $\sig_{\mybr{k(i)}}$ associated to $\OP_{i}$ at iteration $i$. That is, $f(\OP_{i},\sig_{\mybr{k(i)}}) = \tfrac{1}{\vert k(i) \vert} \sum_{j=1}^{\vert k(i) \vert} f(\OP_{i},\sig_{j})$, for $\sig_{j} \in \sig_{\mybr{k(i)}}$. 
With the mean cost for a single batch at hand we are able to calculate the total average including all previous iterations. This reads as $\phi_{i} = \tfrac{1}{i} \sum_{j = 0}^{i-1} f(\OP_{i-j},\sig_{\mybr{k(i-j)}})$. Furthermore, let $\bar{\phi_{i}}$ denote the mean over the last $l$ values of $\phi_i$.
Finally, we are able to state our stopping criterion. The optimization terminates if the relative variation of $\phi_i$, which is denoted as
\begin{align}
    v = \big( \vert \phi_{i} - \bar{\phi_{i}} \vert \big) /  \bar{\phi_{i}},
    \label{eq:stop_citerion}
\end{align}
falls below a certain threshold $\delta$. In our implementation, we set $l = 200$ with a threshold $\delta = 5 \cdot 10^{-5}$.

\subsection{Step size selection}
\label{sec:step_size}
Regarding the convergence rate, a crucial factor of SGD optimization is the selection of the step size (often also referred to as learning rate). For convex problems, the step size is typically based on the Lipschitz continuity property. If the Lipschitz constant is not known in advance, an appropriate learning rate is often chosen by using approximation techniques. In \cite{Roux12astochastic}, the authors propose a basic line search that sequentially halves the step size if the current estimate does not minimize the cost. Other approaches involve some predefined heuristics to iteratively shrink the step size \cite{bottou-tricks-2012} which has the disadvantage of requiring the estimation of an additional hyper-parameter. 
We propose a more variable approach by proposing a variation of a backtracking line search algorithm adapted to SGD optimization.

As already stated in Section~\ref{sec:RuGC}, our goal is to find a set of separable filters such that the empirical sparsity over all $\nsamp$ samples from our training set is minimal. Now, recall that instead of computing the gradient with respect to the full training set, the SGD framework approximates the true gradient by means of a small signal batch or even a single signal sample. That is, the reduced computational complexity comes at the cost of updates that do not minimize the overall objective. However, it is assumed that on average the SGD updates approach the minimum of the optimization problem stated in \eqref{eq:def_sparsepen}, i.e., the empirical mean over all training samples. We utilize this proposition to automatically find an appropriate step size such that for the next iterate an averaging Armijo condition is fulfilled. 

To be precise, starting from an initial step size $a_i^{0}$ the step length $a_i$ is successively shrunk until the next iterate fulfills the Armijo condition. That is, we have
\begin{equation}\begin{split}
    \label{eq:lsearch}
    \bar{f}(\OP_{i+1},\sig_{\mybr{k(i)}}) \leq \bar{f}{}&(\OP_{i},\sig_{\mybr{k(i-1)}})\\
    {}&- a_i \cdot c \cdot \|{\mathbf{G}}(\OP_{i}, \sig_{\mybr{k(i)}}) \|_{F}^{2},
\end{split}\end{equation}
with some constant $c \in (0,1)$. Here, $\bar{f}$ denotes the average cost over a predefined number of previous iterations. The average is calculated over the function values $f(\OP_{i+1},\sig_{\mybr{k(i)}})$, i.e., over the cost of the optimized operator with respect to the respective signal batch. This is achieved via a sliding window implementation that reads
\begin{align}
    \bar{f}(\OP_{i+1},\sig_{\mybr{k(i)}}) = \tfrac{1}{w}  \sum_{j=0}^{w-1} f(\OP_{i+1-j},\sig_{\mybr{k(i-j)}}),
    \label{eq:averaging_window}
\end{align}
with $w$ denoting the window size.

If \eqref{eq:lsearch} is not fulfilled, i.e., if the average including the new sample is not at least as low as the previous average, the step size $a_i$ goes to zero. To avoid needless line search iterations we stop the execution after a predefined number of trials $k_{max}$ and proceed with the next sample without updating the filters and with resetting $a_{i+1}$ to its initial value $a_i^{0}$. The complete step size selection approach is summarized in Algorithm \ref{algo:ls}. In our experiments we set the parameters to $b = 0.9$ and $c = 10^{-4}$.
\newlength{\oldtextfloatsep}\setlength{\oldtextfloatsep}{\textfloatsep}
\begin{algorithm}[t]
\caption{SGD Backtracking Line Search}
\label{algo:ls}
\begin{algorithmic}
\STATE \hspace{-1em}{\textbf{Require:}} $a_i^{0} > 0$, $b \in (0,1)$, $c \in (0,1)$, $\OP_{i}$, \\ $\bar{f}(\OP_{i},\sig_{\mybr{k(i-1)}})$, ${\mathbf{G}}(\OP_{i}, \sig_{\mybr{k(i)}})$, $k_{max} = 40$
\STATE \hspace{-1em}{\textbf{Set:}} $a \leftarrow a_i^{0}$, $k \leftarrow 1$
\WHILE{$\bar{f}(\Gamma(\OP_{i},-{\mathbf{G}}(\OP_{i},\sig_{\mybr{k(i)}}),a),\sig_{\mybr{k(i)}}) >$ \\ $\bar{f}(\OP_{i},\sig_{\mybr{k(i-1)}}) - a \cdot  c \cdot  \|{\mathbf{G}}(\OP_{i}, \sig_{\mybr{k(i)}})\|_{F}^{2} \, \land \, k < k_{max}$}
\STATE $ a \leftarrow b \cdot a$
\STATE $ k \leftarrow k + 1$
\ENDWHILE
\STATE \hspace{-1em}{\textbf{Output:}} $t_{i} \leftarrow a$
\end{algorithmic}
\end{algorithm}

\subsection{Cost Function and Constraints}
An appropriate sparsity measure for our purposes is provided by
\begin{equation}
\label{eq:sparsity_measure}
	g(\coeff) \coloneqq \sum\nolimits_{j = 1}^{\coeffdim} \log\left( 1 + \nu \alpha_k^2 \right).
\end{equation}
This function serves as a smooth approximation to the $\ell_0$-quasi-norm, cf.~\cite{kiechle2014}, but other smooth sparsity promoting functions are also conceivable.

In the section on sample complexity we introduced the oblique manifold as a suitable constraint set. Additionally, there are two properties we wish to enforce on the learned operator as motivated in \cite{hawe:tip13}. (i) Full rank of the operator and (ii) No identical filters. This is achieved by incorporating two penalty functions into the cost function, namely
\begin{align*}
	h(\OP) &= - \tfrac{1}{\sigdim \log(\sigdim)}\log \det \left(\tfrac{1}{\coeffdim}(\OP)^\top \OP\right),\\
	r(\OP) &= - \sum_{k<l} \log \left( 1 - \left((\omegab_{k})^\top (\omegab_{l})\right)^2 \right).
\end{align*}
The function $h$ promotes (i) whereas $r$ enforces (ii).
Hence, The final optimization problem for a set of training samples $\sig_i$ (vectorized versions of signals $\Tsig_i$ in tensor form) is given as
\begin{equation}\begin{split}
	\arg\hspace{-0em}\min_{\OP}\ \tfrac{1}{\nsamp} &\sum_{j=1}^\nsamp f(\OP, \sig_j)\\
	\text{subject to }\quad &\OP = (\OP^{(1)}, \ldots, \OP^{(\tensordim)}),\\ 
	&\OP^{(i)} \in \OB(\coeffdim_i,\sigdim_i),\quad i=1,\ldots,\tensordim, \label{eq:AOL_problem}
\end{split}\end{equation}
with the function
\begin{equation}\begin{split}
\label{eq:costSample}
    f(\OP,\sig) = g\left(\iota(\OP) \sig \right)
	{}&+ \kappa h(\iota(\OP)) + \mu r(\iota(\OP)).
\end{split}\end{equation}
The parameters $\kappa$ and $\mu$ are weights that control the impact of the full rank and incoherence condition. With this formulation of the optimization problem both separable as well as non-separable learning can be handled with the same cost function allowing for a direct comparison of these scenarios.

\section{Experiments}
\label{sec:exp}
The purpose of the experiments presented in this section is, on the one side, to give some numerical evidence of the sample complexity results from Section \ref{sec:sc}, and on the other side,
to demonstrate the efficiency and performance of our proposed learning approach from Section \ref{sec:stochastic}.

\subsection{Learning from natural image patches}
\label{sec:learningfrompatches}

The task of our first experiment is to demonstrate that separable filters can be learned from less training samples compared to learning a set of unstructured filters. We generated a training set that consists of $\nsamp = 500\,000$ two-dimensional normalized samples of size $\Sig_{i} \in \RR^{7 \times 7}$ extracted at random from natural images. 
The learning algorithm then provides two operators $\OP^{(1)},\OP^{(2)} \in \RR^{8 \times 7}$ resulting in $64$ separable filters. We compare our proposed separable approach with a version of the same algorithm, that does not enforce a separable structure on the filters and thus outputs a non-separable analysis operator $\OP \in \RR^{64 \times 49}$.

The weighting parameters for the constraints in \eqref{eq:costSample} are set to $\kappa = 6500$, $\mu = 0.0001$. The factor that controls the slope in the sparsity measure defined in \eqref{eq:sparsity_measure} is $\nu = 500$. At each iteration of the SGD optimization, a batch of $500$ samples is processed. The averaging window size in \eqref{eq:averaging_window} for the line search is fixed to $w = 2000$. In all our experiments we start learning from random filter initializations. 

To visualize the efficiency of the separable learning approach the averaged function value at iteration $i$ as defined in \eqref{eq:averaging_window} is plotted in Figure~\ref{fig:learning}. 
While the dotted curve corresponds to the learning framework that does not enforce a separable structure on the filters, the solid graph visualizes the cost with separability constraint. The improvement in efficiency is twofold. First, imposing separability leads to a faster convergence to the empirical mean of the cost in the beginning of the optimization. Second, the optimization terminates after fewer iterations, i.e., less training samples are processed until no further update of the filters is observed. 
In order to offer an idea of the learned structures, the separable and non-separable filters obtained via our learning algorithm are shown in Figure~\ref{fig:both_operators} as $7 \times 7$ 2D-filter kernels.

\setlength{\textfloatsep}{\oldtextfloatsep}
\begin{figure}[t]
\centering
\includegraphics[width=\columnwidth]{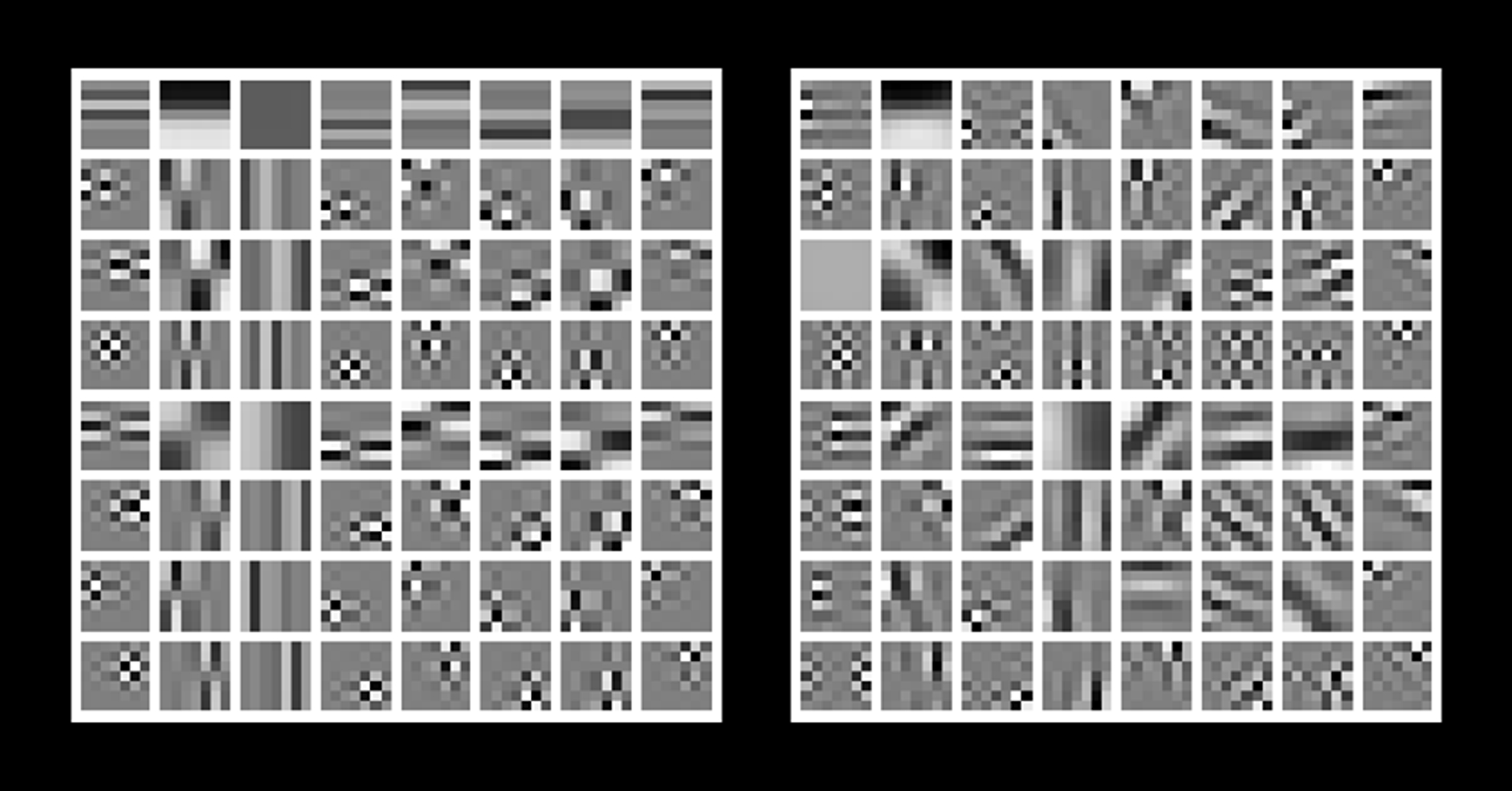}
\caption{Left: Learned filters with separable structure. Right: Result after learning the filters without separability constraint.}
\label{fig:both_operators}
\end{figure}

\begin{figure}[t]
\centering
\includegraphics[width=\columnwidth]{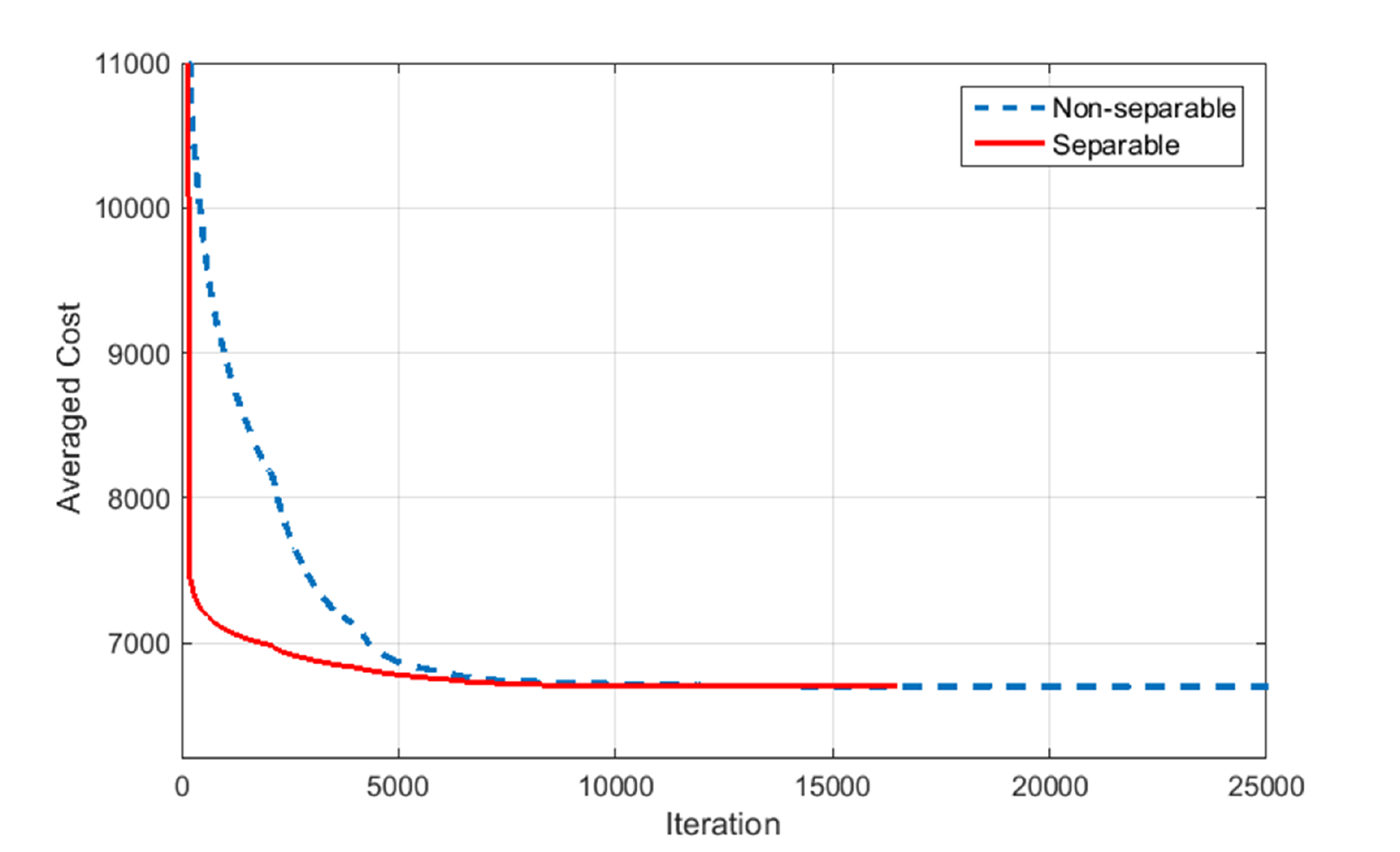}
\caption{Convergence comparison between the SGD based learning framework with and without separability constraint imposed on the filters. The dotted line denotes the averaged cost for the non-separable case. The solid graph indicates the averaged cost when a separable structure is enforced on the filters.}
\label{fig:learning}
\end{figure}

\subsection{Analysis operator recovery from synthetic data}
As mentioned in the introduction, there are many application scenarios where a learned analysis operator can be employed, ranging from inverse problems in imaging, to registration, segmentation and classification tasks. Therefore, in order to provide a task independent evaluation of the proposed learning algorithm, we have conducted experiments that are based on synthetic data and investigated how well a learned operator $\OP_{\text{learned}}$ approximates 
a ground truth operator $\OP_{\text{GT}}$. 
When measuring the accuracy of the recovery we have to take into account that there is an inherent sign and permutation ambiguity in the learned filters.
Hence, we consider the absolute values of the correlation of the filters over all possible permutations.

To be precise, let us denote $\tilde{\omegab}_{i}$ as the $i^{\mathrm{th}}$-row of $\OP_{\text{learned}}$ and $\omegab_{j}$ as the $j^{\mathrm{th}}$-row of $\OP_{\text{GT}}$, both represented as column vectors. We define the deviation of these filters from each other as $c_{ij} = 1 - \vert \tilde{\omegab}_{i}^{\top} \omegab_{j}  \vert$. Doing this for all possible combinations of $i$ and $j$ we obtain the confusion matrix  ${\mathbf{C}}$, where the $(i,j)$-entry $\mathbf{C}$ is $0$ if $\tilde{\omegab}_{i}$ is equal to $\omegab_{j}$. Building the confusion matrix accounts for the permutation ambiguity between $\OP_{\text{GT}}$ and $\OP_{\text{learned}}$.
Next, we utilize the Hungarian-method \cite{kuhn1955} to determine the path through the confusion matrix ${\mathbf{C}}$ with the lowest accumulated cost under the constraint that each row and each column is visited only once. 
In the end, the coefficients along the path are accumulated and this sum serves as our error measure denoted as $H({\mathbf{C}})$.
In other words, we aim to find the lowest sum of entries in $\mathbf{C}$ such that in each line a single entry is picked and no column is used twice.
With this strategy we prevent that multiple retrieved filters $\tilde{\omegab}_{i}$ are matched to the same filter $\omegab_{j}$, i.e., the error measure $H({\mathbf{C}})$ is zero if and only if all filters in $\OP_{\text{GT}}$ are recovered.

Following the procedure in \cite{Nam201330}, we generated a synthetic set of samples of size $\Sig_{i} \in \RR^{7 \times 7}$ w.r.t.\ to $\OP_{\text{GT}}$. As the ground truth operator we chose the separable operator obtained in the previous Subsection~\ref{sec:learningfrompatches}.
The generated signals exhibit a predefined co-sparsity after applying the ground truth filters to them. The set of samples has the size $\nsamp = 500\,000$. The co-sparsity, i.e., the number of zero filter responses is fixed to $15$. Additive white Gaussian noise with standard deviation $0.05$ is added to each normalized signal sample. We now aim at retrieving the underlying original operator that was used to generate the signals. Again, we compare our separable approach against the same framework without the separability constraint. 

In order to compare the performance of the proposed SGD algorithm in the separable and non-separable case, we conduct an experiment where the size of the training sample batch that is used for the gradient and cost calculation is varied.
The employed batch sizes are $\{1,10,25,50,75,100,250,500,1000\}$, while the performance is evaluated over ten trials, i.e., ten different synthetic sets that have been generated in advance. 

Figure~\ref{fig:SampleComplexityCosp15} summarizes the results for this experiment. For each batch size the error over all 10 trials is illustrated. The left box corresponds to the separable approach and accordingly the right box denotes the error for the non-separable filters. While the horizontal dash inside the boxes indicates the median over all 10 trials the boxes represent the mid-$50\%$. The dotted dashes above and below the boxes indicate the maximum and minimum error obtained.

It is evident that the separable operator learning algorithm achieves better recovery of the ground truth operator for smaller batch sizes, which indicates that it requires less samples in order to produce good recovery results.
Table~\ref{tab:iterations} shows the average number of iterations until convergence and the averaged error over all trials. As can be seen from the table, the separable approach requires less samples to achieve good accuracy, has a faster convergence and a smaller recovery error compared to the non-separable method. 
As an example, the progress of the recovery error for a batch size of $500$ is plotted in Figure~\ref{fig:OpRecError}. 

Finally, in order to show the efficiency of the SGD-type optimization we compare our method to an operator learning framework that utilizes a full gradient computation at each iteration. We chose the algorithm proposed in \cite{hawe:tip13} which learns a non-separable set of filters via a geometric conjugate gradient on manifolds approach. Again, we generated ten sets of $\nsamp = 500\,000$ samples with a predefined cosparsity. Based on this synthetic set of signals, we measured the mean computation time and number of iterations until the stopping criterion from Section~\ref{subsec_gsgd} is fulfilled. Table \ref{tab:sgdvscg} summarizes the results for the proposed separable SGD, the non-separable SGD and the non-separable CG implementations. While the CG based optimization converges after only a few iterations, the overall execution time is worse compared to SGD due to the high computational cost for each full gradient calculation. 

Figure~\ref{fig:OpRecError} and Table~\ref{tab:sgdvscg} in particular support the theoretical results obtained for the sample complexity. They illustrate that the separable SGD algorithm requires less iterations, and therefore fewer samples, to reach the dropout criterion compared to the SGD algorithm that does not enforce separability of the learned operator.

\begin{table}[t]
\small
\renewcommand{\arraystretch}{1.3}
\caption{Comparison between the convergence speed and average error for the filter recovery experiment. For each batch size the average number of iterations until convergence is given along with the average error which denotes the mean of the values for $H({\mathbf{C}})$ over all the ten trials. Upper part: Learning separable filters. Lower part: Learning non-Separable filters.}
\label{tab:iterations}
\centering
\tiny
\begin{tabular}{c||c|c|c|c|c|c|c|c|c}

\hline
	     	& 1	& 10 & 25 &	50 & 75 & 100 &	250	& 500 &	1000 \\
\hline
										
Iterations	&	450	  & 422	   & 2573	& 3721	& 4595 & 5780 &	9673  &	13637 & 15198 \\
Avg. error	&	37.22 &	33.28  & 1.39	& 1.08	& 0.75 & 0.56 &	0.31  & 0.24  & 0.14 \\
\hline
\hline
Iterations	&	1812  & 1800 &	2786  &	3866  &	10147 &	13595 &	19679 &	21087 &	20611  \\
Avg. error	&	41.21 & 40.73 &	38.68 &	27.74 &	9.67  &	1.89  &	1.38  &	1.23 &	1.14  \\

\hline
\end{tabular}
\vspace{-3pt}
\end{table}

\begin{figure}[t]
\centering
\includegraphics[width=\columnwidth]{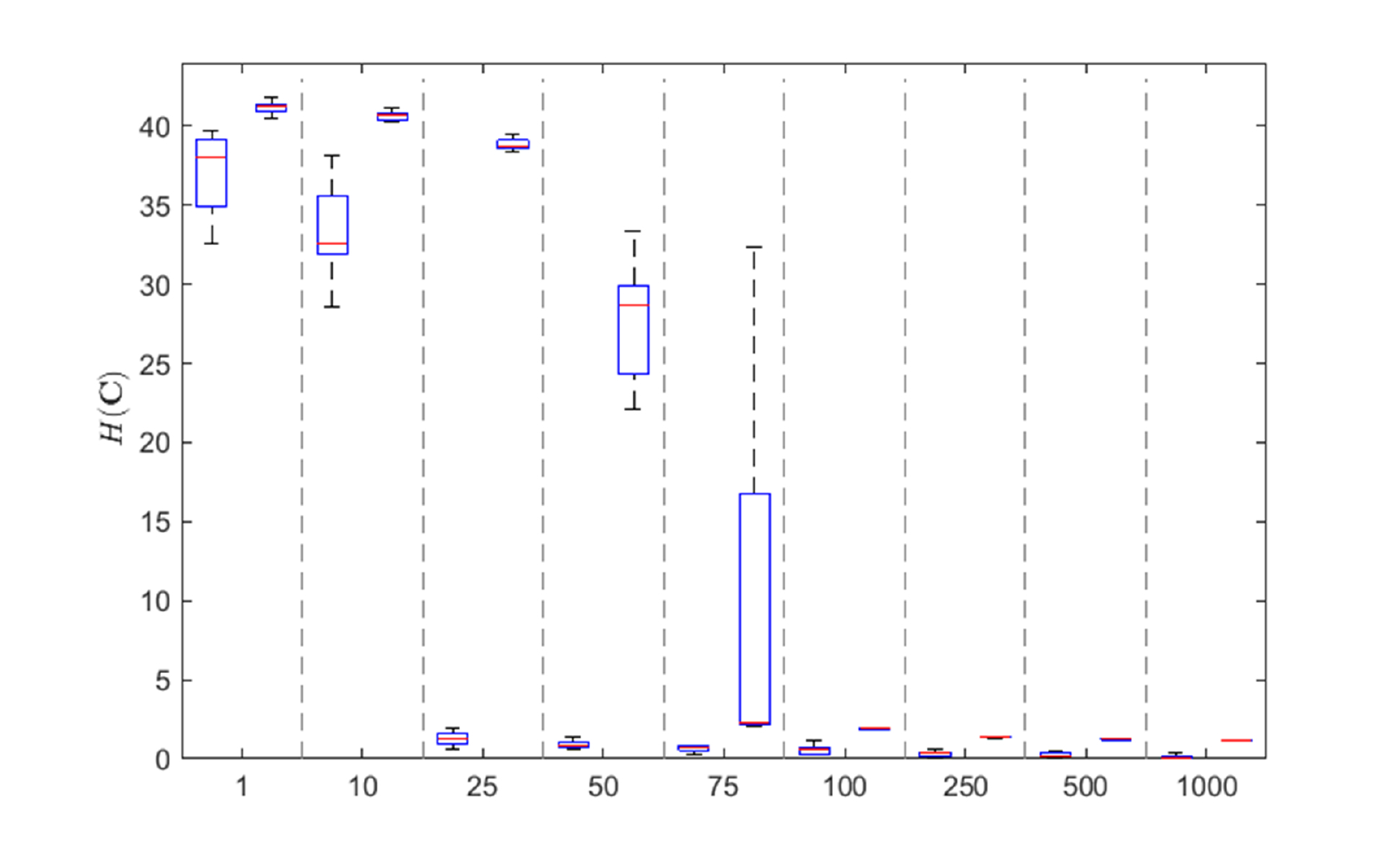}
\caption{
The horizontal axis indicates the batch size. For each size, on the left the recovery error for the separable case is plotted, whereas the error for the non-separable case is plotted on the right. All generated signals exhibit a co-sparsity of 15 and for each batch size the error is evaluated over 10 trials. The horizontal dash inside the boxes indicates the median error over all trials. The boxes represent the central 50\% of the errors obtained, while the dashed lines above and below the boxes indicate maximal and minimal errors.}
\label{fig:SampleComplexityCosp15}
\end{figure}

\begin{figure}[t]
\centering
\includegraphics[width=\columnwidth]{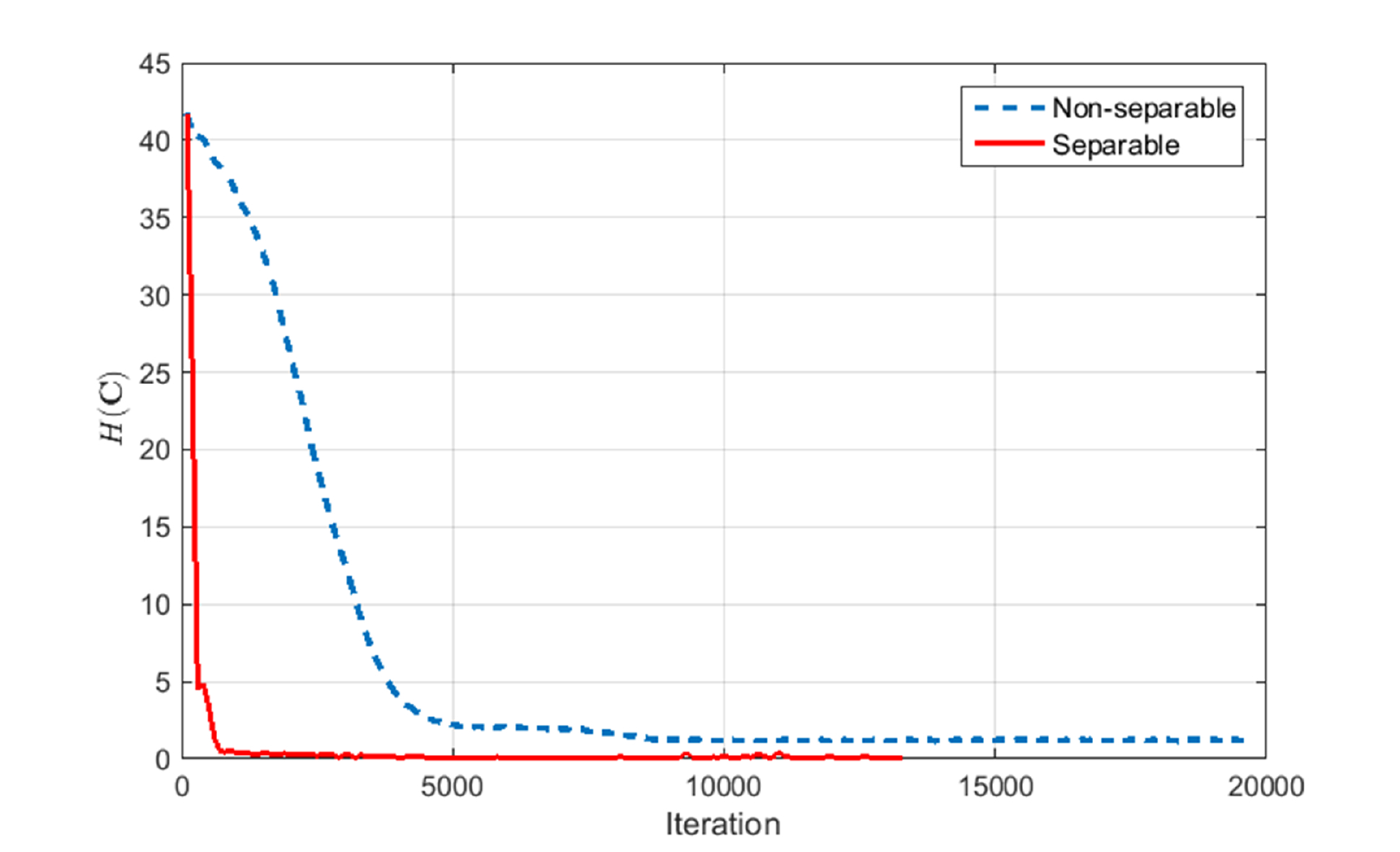}
\caption{Operator recovery error for a batch size of $500$. The dotted graph indicates the progress over the iterations for the non-separable case, while the solid curve shows the progress for retrieving separable filters.}
\label{fig:OpRecError}
\end{figure}

\begin{table}[t]
\footnotesize
\renewcommand{\arraystretch}{1.3}
\caption{Comparison between the SGD optimization and conjugate gradient optimization. Average number of iterations and processing times over ten trials.}
\label{tab:sgdvscg}
\centering
\begin{tabular}{c||c|c|c}

\hline
	     	& Iterations & time in sec & $H(\mathbf{C})$  \\
\hline
SGD (separable)      & 12358   & 1176   & 0.48   \\
SGD (non-separable)	    & 21660    & 1554    & 1.24  \\
CG (GOAL \cite{hawe:tip13}) & 601    & 2759    & 1.01\\
\hline

\end{tabular}
\vspace{-3pt}
\end{table}

\subsection{Comparison with related approaches on image data} 

In order to show that the operator learned with separable structures is applicable to real world signal processing tasks, we have conducted a simple image denoising experiment. 
Rather than outperforming existing denoising algorithms, the message conveyed by this experiment is that using separable filters only slightly reduces the reconstruction performance. 
We compare our separable operator (sepSGD) against other learning schemes that provide a set of filters without imposing a separability constraint. 
Specifically, in addition to our non-separable SGD implementation (SGD), we have chosen the geometric analysis operator learning scheme (GOAL) from \cite{hawe:tip13}, the Analysis-KSVD algorithm (AKSVD) proposed in \cite{rubinstein2013analysis}, and the method presented in \cite{yaghoobi2013constrained} (CAOL) for comparison.
All operators are learned from $\nsamp = 500\,000$ patches of size $7 \times 7$ extracted from eight different standard natural training images that are not included in the test set. All operators are of size $64 \times 49$. For sepSGD, SGD and GOAL we have set the parameters to the same values as already stated in \ref{sec:learningfrompatches}. The parameters of the AKSVD and CAOL learning algorithms have been tuned such that the overall denoising performance is best for the chosen test images.
Four standard test images (Barbara, Couple, Lena, and Man), each of size $512 \times 512$ pixels, have been artificially corrupted with additive white Gaussian noise with standard deviation $\sigma_n \in \{10,20,30\}$.

The learned operators are used as regularizers in the denoising task which is formulated as an inverse problem.
We have utilized the NESTA algorithm \cite{becker:siam2011} which solves the analysis-based unconstrained inverse problem
\begin{equation*}
 {\mathbf{x}}^{\star} \in \arg\min_{{\mathbf{x}} \in \RR^{n}} \tau \| \OP^{\ast}({\mathbf{x}}) \|_{1} + \tfrac{1}{2} \| {\mathbf{y}} - {\mathbf{x}} \|_{2}^{2},
\end{equation*}
where ${\mathbf{x}}$ represents a vectorized image, ${\mathbf{y}}$ are the noisy measurements, $\tau$ is a weighting factor, and $\OP^{\ast}({\mathbf{x}})$ denotes the operation of applying the operator $\OP^{\ast}$ to all overlapping patches of the image ${\mathbf{x}}$. This is done by applying each of the learned filters to the patches via convolution.
For all operators, the weighting factor is set to $\tau \in \{0.18, 0.40, 0.60\}$ for the noise levels $\sigma_n \in \{10,20,30\}$, respectively.
Table \ref{tab:denoising} summarizes the results of this experiment. 

The presented results indicate that using separable filters does not reduce the image restoration performance to a great extent and that separable filters are competitive with non-separable ones. Furthermore, the SGD update has the advantage that the execution time of a single iteration in the learning phase does not grow with the size of the training set which is an important issue for extensive training set dimensions or online learning scenarios.
Please note that we have not optimized the parameters of our learning scheme for the particular task of image denoising.

\begin{table}[t]
\footnotesize
\renewcommand{\arraystretch}{1.3}
\caption{Denoising experiment for four different test images corrupted by three noise levels. Achieved PSNR in decibels (dB).}
\label{tab:denoising}
\centering
\begin{tabular}{c|c||c|c|c|c}

\hline
$\sigma_n$ / PSNR  &     & Barbara & Couple  & Lena    & Man      \\
\hline
\hline 
              & sepSGD  & 32.13   & 32.76   & 34.10   & 32.71    \\ 
              & SGD     & 31.82   & 32.43   & 33.98   & 32.64    \\ 
10 / 28.13    & GOAL    & 32.28   & 32.62   & 34.29   & 32.80    \\
              & AKSVD   & 31.75   & 31.69   & 33.51   & 31.97    \\
              & CAOL    & 30.44   & 30.35   & 31.41   & 30.72    \\
\hline
              & sepSGD  & 27.89   & 28.97   & 30.46   & 29.00    \\ 
              & SGD     & 27.61   & 28.69   & 30.36   & 28.97    \\ 
20 / 22.11    & GOAL    & 28.01   & 28.86   & 30.65   & 29.12    \\
              & AKSVD   & 27.49   & 27.68   & 29.85   & 28.22    \\
              & CAOL    & 26.05   & 26.23   & 27.27   & 26.63    \\
               
\hline
              & sepSGD  & 25.64   & 26.83   & 28.24   & 27.02    \\  
              & SGD     & 25.43   & 26.63   & 28.22   & 27.02    \\ 
30 / 18.59    & GOAL    & 25.75   & 26.76   & 28.40   & 27.12    \\
              & AKSVD   & 25.37   & 25.66   & 27.75   & 26.33    \\
              & CAOL    & 23.72   & 24.00   & 24.89   & 24.35    \\
\hline

\end{tabular}
\vspace{-3pt}
\end{table}

\section{Conclusion}
We proposed a sample complexity result for analysis operator learning for signal distributions within the unit $\ell_2$-ball, where we have assumed that the sparsity promoting function fulfills a Lipschitz condition.
Rademacher complexity and McDiarmid's inequality were utilized to prove that the deviation of the empirical co-sparsity of a training set and the expected co-sparsity is bounded by $\mathcal{O}(C / \sqrt{\nsamp})$ with high probability, where $\nsamp$ denotes the number of samples and $C$ is a constant that among other factors depends on the (separable) structure imposed on the analysis operator during the learning process.
Furthermore, we suggested a geometric stochastic gradient descent algorithm that allows to incorporate the separability constraint during the learning phase. An important aspect of this algorithm is the line search strategy which we designed in such a way that it fulfills an averaging Armijo condition.
Our theoretical results and our experiments confirmed that learning algorithms benefit from the added structure present in separable operators in the sense that fewer training samples are required in order for the training phase to provide an operator that offers good performance. 
Compared to other co-sparse analysis operator learning methods that rely on updating the cost function with respect to a full set of training samples in each iteration, our proposed method benefits from a dramatically reduced training time, a common property among SGD methods. This characteristic further endorses the choice of SGD methods for co-sparse analysis operator learning.

\appendix
\newcommand{\sigt}{\mathbf{\tilde{s}}}
\newcommand{\Sigt}{\mathbf{\tilde{S}}}

\section*{Addition to proof of Lemma~\ref{lem:dist_with_Gauss}:}
In order to upper bound the expectation of $\Phi(\Sig)$ in \eqref{eq:lem_7_proof_1}, we follow a common strategy which we outline in the following for the convenience of the reader. First, we introduce a set of ghost samples $\Sigt = [\sigt_1, \ldots, \sigt_{\nsamp}]$ where all samples are drawn independently according to the same distribution as the samples in $\Sig$. For this setting the equations $\EE_{\Sigt}[\hat{\EE}_{\Sigt} [f]] = \EE[f]$ and $\EE_{\Sigt}[\hat{\EE}_{\Sig} [f]] = \hat{\EE}_{\Sig}[f]$ hold. Using this, we deduce
\begin{align}
    \EE&_\Sig[\Phi(\Sig)] 
    = \EE_{\Sig} \bigg[\sup_{f \in \Ffrak} \EE_{\Sigt} [\tfrac{1}{\nsamp} \sum\nolimits_{i}(f(\OP, \sigt_i) - f(\OP, \sig_i))] \bigg] \nonumber\\
    &\leq \EE_{\Sig, \Sigt} \bigg[\sup_{f \in \Ffrak} \tfrac{1}{\nsamp}\sum\nolimits_{i}(f(\OP, \sigt_i) - f(\OP, \sig_i)) \bigg] \label{eq:app_jens}\\
    &= \EE_{\sigma,\Sig, \Sigt} \bigg[\sup_{f \in \Ffrak} \tfrac{1}{\nsamp}\sum\nolimits_{i}\sigma_i(f(\OP, \sigt_i) - f(\OP, \sig_i)) \bigg] \label{eq:app_expsig}\\
    &\leq 2\Rad{\Ffrak}. \nonumber
\end{align}
Here, the inequality \eqref{eq:app_jens} holds because of the convexity of the supremum and by application of Jensen's inequality, \eqref{eq:app_expsig} is true since $\EE[\sigma_i] = 0$ and the last inequality follows from the definition of the supremum and using the fact that negating a Rademacher variable does not change its distribution.

The next step is to bound the Rademacher complexity by the empirical Rademacher complexity. To achieve this, note that $\empRad{\Ffrak}$, like $\Phi$, fulfills the condition for McDiarmid's theorem with factor $2\lambda\sqrt{\coeffdim}/\nsamp$. This leads to the final result.


\section*{Acknowledgment}

This work was supported by the German Research Foundation (DFG) under grant KL 2189/8-1. The contribution of Remi Gribonval was supported in part by the European Research Council, PLEASE project (ERC-StG-2011-277906).

\ifCLASSOPTIONcaptionsoff
  \newpage
\fi

\bibliographystyle{IEEEtran}
\bibliography{references}

\end{document}